\documentclass{article}

\usepackage{hyperref}

\usepackage{microtype}
\usepackage{color}

\usepackage{graphicx}
\usepackage{subfigure}
\usepackage{bm}
\usepackage{amsmath,cleveref,autonum}
\usepackage{amsthm}
\usepackage{amssymb}
\usepackage{booktabs} % for professional tables
\usepackage{algorithm}
\usepackage{mathtools}

\usepackage[accepted]{icml2018}

\graphicspath{ {images/} }

\icmltitlerunning{ Fourier Policy Gradients }

\DeclareMathOperator{\E}{\mathbb{E}}
\newcommand{\FT}[1]{ \mathcal{F}\left(#1\right)} 
\newcommand{\IFT}[1]{ \mathcal{F}^{-1}\left( #1 \right)} 

\newcommand{\IP}[0]{ {\hat{I} }}
\newcommand{\BO}[0]{ {\bm{\omega} }}
\newcommand{\BA}[0]{ {\bm{a}}}
\newcommand{\BF}[0]{ {\bm{f} }}
\newcommand{\BMU}[0]{ {\bm{\mu} }}
\newcommand{\BSIGMA}[0]{ {\bm{\Sigma} }}

\newcommand{\BL}[0]{ {\bm{l} }}
\newcommand{\BS}[0]{ {\bm{S} }}

\newtheorem{theorem}{Theorem}
\newtheorem{corollary}{Corollary}[theorem]
\newtheorem{lemma}{Lemma}

\theoremstyle{definition}

\newtheorem{innercustomgeneric}{\customgenericname}
\providecommand{\customgenericname}{}
\newcommand{\newcustomtheorem}[2]{%
  \newenvironment{#1}[1]
  {%
   \renewcommand\customgenericname{#2}%
   \renewcommand\theinnercustomgeneric{##1}%
   \innercustomgeneric
  }
  {\endinnercustomgeneric}
}

\newcustomtheorem{customthm}{Theorem}
\newcustomtheorem{customlemma}{Lemma}
\newcustomtheorem{customcorollary}{Corollary}

\begin{document}

\twocolumn[
\icmltitle{Fourier Policy Gradients}



\icmlsetsymbol{equal}{*}

\begin{icmlauthorlist}
\icmlauthor{Matthew Fellows}{equal,ox}
\icmlauthor{Kamil Ciosek}{equal,ox}
\icmlauthor{Shimon Whiteson}{ox}
\end{icmlauthorlist}

\icmlaffiliation{ox}{Department of Computer Science, University of Oxford, United Kingdom}

\icmlcorrespondingauthor{Matthew Fellows}{matthew.fellows@cs.ox.ac.uk}

\icmlkeywords{Reinforcement Learning, MPDs, Fourier Series, Control, Policy Gradients, Machine Learning, ICML}

\vskip 0.3in
]



\printAffiliationsAndNotice{\icmlEqualContribution} % otherwise use the standard text.

\begin{abstract}
We propose a new way of deriving policy gradient updates for reinforcement learning. Our technique, based on Fourier analysis, recasts integrals that arise with \emph{expected policy gradients} as convolutions and turns them into multiplications. The obtained analytical solutions allow us to capture the low variance benefits of EPG in a broad range of settings. For the critic, we treat trigonometric and radial basis functions, two function families with the universal approximation property. The choice of policy can be almost arbitrary, including mixtures or hybrid continuous-discrete probability distributions. Moreover, we derive a general family of sample-based estimators for stochastic policy gradients, which unifies existing results on sample-based approximation. We believe that this technique has the potential to shape the next generation of policy gradient approaches, powered by analytical results.

\end{abstract}

\section{Introduction}

\emph{Policy gradient} methods, also known as actor-critic methods, are an effective way to perform reinforcement learning in large or continuous action spaces \cite{lillicrap2015continuous, schulman2015trust, schulman2017proximal, wu2017scalable, peters2008reinforcement, spg, reinforce}. Since they adjust the policy in small increments, they do not have to perform expensive optimisations over the action space, in contrast to methods based on value functions, such as $Q$-learning \citep{dqn, alphago, doubleq} or SARSA \cite{van2009theoretical, sutton, sutton1996generalization}. Moreover, they are naturally suited to stochastic policies, which are useful for exploration and necessary to achieve optimality in some settings, e.g., competitive multi-agent systems.

Until recently, policy gradient methods were either restricted to deterministic policies \citep{dpg} or suffered from high variance \citep{spg}. The latter problem is exacerbated in large state spaces, when the number of samples required to reduce the variance of the gradient estimate becomes infeasible for the simple score function estimators on which policy gradient methods typically rely.  The problem also arises when training recurrent neural networks (RNNs) \citep{seq1, seq2, learning_multi_comm} that have to be unrolled over several timesteps, each adding to the overall variance, and in multi-agent settings \citep{coma}, where the actions of other agents introduce a compounding source of variance.

Recently, a new approach called \emph{expected policy gradients} (EPG) \citep{epg, epg-journal} was proposed that eliminates the variance of a stochastic policy gradient by integrating over the actions analytically. However, this requires analytic solutions to the policy gradient integral and the original work addressed only polynomial critics. 

In this paper, we employ techniques from Fourier analysis to derive analytic policy gradient updates for two important families of critics. The first, \emph{radial basis functions} (RBFs), combines the benefits of shallow structure, which makes them tractable, with an impressive empirical track record \citep{buhmann2003radial, carr2001reconstruction}. The second, trigonometric critics, is useful for modelling periodic phenomena. Similarly to polynomial critics \citep{epg-journal}, these function classes are universal, i.e., they can approximate an arbitrary function on a bounded interval. 

Furthermore, to address cases where analytical solutions are infeasible, we provide a general theorem for deriving Monte Carlo estimators that links existing methods using the first and second derivatives of the action-value function, relating it to existing sampling approaches.  

Our technique also enables analytic solutions for new families of policies, extending EPG to any policy that has unbounded support, where it previously required the policy to be in an exponential family. We also develop results for mixture policies and hybrid discrete-continuous policies, which we posit can be useful in multi-agent settings, where having a rich class of policies is important not just for exploration but also for optimality \citep{nisan2007algorithmic}.

Overall, we believe that the techniques developed in this paper can be used to shape the next generation of policy gradient methods suitable for any reasonable MDP and that, powered by analytical results, achieve zero or low variance. Moreover, our methods elucidate the way policy gradients work by explicitly stating the expected update. Finally, while the main contribution of this paper is theoretical, we also provide an empirical evaluation using a periodic critic on a simple turntable problem that demonstrates the practical benefit of using a trigonometric critic.

\section{Background}
\label{sec-bckg}

\emph{Reinforcement learning} (RL) aims to learn optimal behaviour policies for an agent (or many agents) acting in an environment with a scalar reward signal. Formally, we consider a Markov decision process, defined as a tuple $(S,A,R,p,p_0,\gamma)$. An agent has an environmental state $\bm{s} \in S = \mathbb{R}^n$; takes a sequence of actions $\bm{a}_1,\bm{a}_2,...$,  where $\bm{a}_t \in A$; transitions to the next state $\bm{s}' \sim p(\cdot|\bm{s},\bm{a})$ under the state transition distribution $p(\bm{s}' \vert \bm{s}, a)$; and receives a scalar reward $r \in \mathbb{R}$. The agent's initial state $\bm{s}_0$ is distributed as $\bm{s}_0 \sim p_0(\cdot)$.  

The agent samples from the policy $\beta$ to generate actions $a\sim\beta(\cdot|\bm{s})$, giving a trajectory through the environment $\tau=(\bm{s}_0,\bm{a}_0,r_1,\bm{s}_1,\bm{a}_1,r_1,...)$. The definition of the value function is $ V^{\beta}(\bm{s}) = \E_{\tau : \bm{s}_0 = \bm{s}}[\sum_{t}\gamma^t r_t ]$ and action-value function is $Q^{\beta}(\bm{s},\bm{a}) = \E_{\tau : \bm{s}_0 = \bm{s}, a_0 = a}[\sum_{t}\gamma^t r_t]$, where $\gamma \in [0,1)$ is a discount factor. An optimal policy $\beta^\star$ maximises the total return $J=\int_{\bm{s}} V^{\beta^{\star}}(\bm{s}) dp_o(\bm{s})$. %Finding an optimal policy, also known as \emph{solving} the MDP, is the aim of reinforcement learning \citep{sutton, rl_algorithms}.

\subsection{Policy Gradient Methods}
Policy gradient methods seek a locally optimal policy by maintaining a \emph{critic}, learned using a value-based method, and an \emph{actor}, adjusted using a policy gradient update.  

The critic $\hat{Q}$ is learned using variants of SARSA \cite{van2009theoretical, sutton, sutton1996generalization}, with the goal of approximating the true action-value $Q^{\beta}$. Meanwhile, the actor adjusts the policy parameter vector $\theta$ of the policy $\beta_\theta$ with the aim of maximising $J$. For stochastic policies, this is done by following the gradient:
\begin{gather}
\label{eq:policy_objective1} \nabla_{\theta}J =\int_s{\rho(\bm{s}) \underbrace{ \int_a{Q(\bm{s},\bm{a})\nabla_{\theta}\beta_{\Theta}(\bm{a}|\bm{s})}d\bm{a}} _{I_\theta(s)}}
d\bm{s},
\end{gather}
where $\rho(\bm{s})\triangleq\sum_{t=0}^{\infty}\gamma^tp(\bm{s}_t=\bm{s}|\bm{s}_0)$ is the discounted-ergodic occupancy measure.
The outer integral can be approximated by following a trajectory of length $T$ through the environment, yielding:
\begin{gather}
\label{eq:policy_objective2} 
\nabla_{\theta} J = 
\E_{\rho(s)} [ I_{\theta}(\bm{s}) ] \approx \sum_{t=0}^{T-1}\gamma^t \hat{I_{\theta}}(\bm{s}_t),
\end{gather}
where $\hat{I_{\theta}}$ is the integral of \eqref{eq:policy_objective1} but with the critic $\hat{Q}$ in place of the unknown true $Q$-function:
\[
\label{eq:rl_integral_objective}
\hat{I_{\theta}}(\bm{s}_t)=\int_a{\hat{Q}(\bm{s}_t,\bm{a})\nabla_{\theta}\beta_{\Theta}(\bm{a}|\bm{s}_t)}d\bm{a}.
\]

The subscript of $\hat{I}_{\theta}$ denotes the fact that we are differentiating with respect to $\theta$. Now, since $\hat{Q}$, unlike $Q$, does \emph{not} depend on the policy parameters, we can move the differentiation out of the inner integral as follows:
\begin{gather}
\label{eq-iqm}
\hat{I_{\theta}}(\bm{s}_t)= \nabla_{\theta} \underbrace{\int_{a}{\hat{Q}(\bm{s}_t,\bm{a}) d \beta_{\Theta}(\bm{a}|\bm{s}_t)}}_{ E(\bm{s_t}) = \mathbb{E}_{\beta}[\hat{Q}(\bm{s_t}, \cdot)] }. 
\end{gather}
This transformation has two benefits: it allows for easier manipulation of the integral and it also holds for deterministic policies, where $\beta$ is a Dirac-delta measure \cite{dpg, epg, epg-journal}. 

Using \eqref{eq:policy_objective2} directly with an analytic value of $\hat{I}(s_t)$ yields \emph{expected policy gradients} (EPG)\footnote{Called \emph{all-action policy gradient} in an unpublished draft by \citet{sutton2000comparing}.} \citep{epg, epg-journal}, shown in Algorithm \ref{epg-algo}. If instead we add an additional Monte Carlo sampling step:
\[
\label{spg-mc}
\hat{I}_{\theta} \approx \hat{Q}(\bm{s},\bm{a}) \nabla_{\theta} \log \beta(\bm{a} \vert \bm{s}),
\]
we get the original \emph{stochastic policy gradients} \citep{reinforce, spg}. In place of \eqref{spg-mc}, alternative Monte Carlo schemes with better variance properties have also been proposed \citep{baxter2000direct, baxter2001experiments,baxter2001infinite, qprop, kakade2003sample}. If we can compute the integral in \eqref{eq:rl_integral_objective}, then EPG is preferable since it avoids the variance introduced by the Monte Carlo step of  \eqref{spg-mc}.

\begin{algorithm}[tb]
\begin{algorithmic}[1]
\STATE $s \gets s_0$, $t \gets 0$
\STATE initialise optimiser, initialise policy parameters $\theta$
\WHILE{not converged}
\STATE $g_t \gets \gamma^t \IP_\theta(s)$ 
\STATE $\theta \gets  \theta \; + \;  $optimiser.\textsc{update}$(g_t) $
\STATE $a \sim \beta(\cdot, s)$
\STATE $s',r \gets $ environment.\textsc{perform-action}(a)
\STATE $\hat{Q}$.\textsc{update}($s,a,r,s'$)
\STATE $t \gets t + 1$
\STATE $s \gets s'$
\ENDWHILE
\end{algorithmic}
\caption{Expected Policy Gradient}
\label{epg-algo}
\end{algorithm}

This paper considers both methods for solving integrals of the form in \eqref{eq:rl_integral_objective} and Monte Carlo methods that improve on \eqref{spg-mc}  for cases where analytical solutions are not possible. 

Above, we used the symbol $\theta$ to denote a generic policy parameter. Often, the policy is described by its moments (for instance a Gaussian is fully defined by its mean and covariance). To achieve greater flexibility, these immediate parameters are obtained by a complex function approximator, such as a neural network, where the state vector is the input, and parameterised by $\bm{w}$. The total policy gradient for $\bm{w}$ is then obtained by using the chain rule. For example, for a Gaussian we have immediate parameters $ \bm{\mu},\bm{\Sigma}$ and the parameterisation is:
\begin{gather}
\beta(\bm{a} \vert \bm{s}) = \mathcal{N}(\bm{\mu}, \bm{\Sigma}), \\
(\bm{\mu}, \bm{\Sigma}^{1/2}) = \text{net}_{\bm{w}}(\bm{s}),
\end{gather}
where $\text{net}_{\bm{w}}$ is a neural network parameterised by the vector $\bm{w}$. The gradient for some $\bm{w}$ is then:
\begin{gather}
\nabla_{\bm{w}} I_\theta(\bm{s}) = \nabla_{\bm{w}} \bm{\mu} I_{\bm{\mu}} + \nabla_{\bm{w}} \bm{\Sigma}^{1/2} I_{\bm{\Sigma}^{1/2}}.
\end{gather}
For clarity, we only give updates for the immediate parameters (in this case, $\IP_{\bm{\mu}}$ and $\IP_{\bm{\Sigma^{1/2}}}$) in the remainder of the paper, without explicitly mentioning $\bm{w}$. 

\subsection{Fourier Analysis}
\label{sec-fa}
A \emph{convolution} $f * g$ is an operation on two functions that returns another function, defined as:
\[
(f*g)(\bm{x}) \triangleq \int_{x'} f(\bm{x}')g(\bm{x}-\bm{x}')d\bm{x}'.
\]
Convolutions have convenient analytical properties that we use to derive our main result. To make convolutions easy to compute, we seek a transform $\mathcal{F}$ that, when applied to a convolution, yields a simple operation like multiplication, i.e., we want the property:
\begin{gather}
\label{ft-convolution}
\FT{f * g}(\bm{\omega}) = \Big(\FT{f} \FT{g}\Big)(\bm{\omega}).
\end{gather}
We also need the dual property:
\[\FT{fg} = \FT{f} * \FT{g},\]
to ensure symmetry between the space of functions and their transforms.  It turns out that, up to scaling, there is only one transform that meets our needs \citep{alesker2008characterization},  the \emph{Fourier transform}:
\[
\FT{f}(\bm{\omega}) \triangleq \int_{x} f(\bm{x})e^{-i \bm{\omega}^\top \bm{x}}d\bm{x}.
\]
The two sets of parentheses on the lefthand side are required because the Fourier transform $\FT{f}$ is a function, not a scalar, and $\FT{f}(\bm{\omega})$ is the result of evaluating this function on $\bm{\omega}$. The Fourier transform of a probability density function is known as the \textit{characteristic function} of the corresponding distribution. An intuitive interpretation of the Fourier transforms is that it provides a mapping from the action-spatial domain to the frequency domain, $\FT{f(\bm{x})}:\bm{x}\rightarrow\bm{\omega}$, decomposing the function $f$ into its frequency components. Consider, as a simple example, a univariate sinusoidal function, $f(x)=\cos(x\Omega)$. The Fourier transform of $f$ can be easily shown to be $\FT{f}=\pi\delta(\omega-\Omega)+\pi\delta(\omega+\Omega)$ \cite{fourieranalysis}; the Fourier transform has mapped a sinusoid of frequency $\Omega$ in the action domain to a double frequency spike at $\pm \Omega$ in the frequency domain.

The Fourier transform has another related intuitive interpretation as a change of basis in the space of functions. The Fourier basis functions $e^{-i \bm{\omega}^\top \bm{x}}$ make analytical operations convenient in much the same way as a choice of convenient basis in linear algebra makes certain matrix operations easier. Since the basis functions are periodic, the Fourier transform can also be viewed as a decomposition of the original function into cycles. Sometimes these cycles are written explicitly when the complex exponential $e^{-i \bm{\omega}^\top \bm{x}}$ is expressed in polar form, which includes sines and cosines. Indeed, the \emph{Fourier series}, which we briefly discuss in Appendix \ref{fseries}, can be used to prove that any function on a bounded interval can be approximated arbitrarily well with a sum of sufficiently many such trigonometric terms.     
   
The \emph{inverse Fourier transform} is defined as: 
\[
\IFT{g}(\bm{x}) \triangleq \frac{1}{(2\pi)^n} \int_{\omega} g(\bm{\omega}) e^{i \bm{\omega}^\top \bm{x}} d\bm{\omega},
\]
which has the property that, for any function $f$,
\begin{gather}
\label{eq-ft-inv-ft}
\IFT{\FT{f}} = f.
\end{gather}
Thus, we can recover the original function by applying the Fourier and inverse Fourier transforms. Just as the Fourier transform maps from the action domain to the frequency domain, the inverse Fourier transforms provides a mapping from the frequency domain back to the action-spatial domain $\IFT{f(\bm{\omega})}:\bm{\omega}\rightarrow\bm{x}$.
The Fourier transform also turns differentiation into multiplication:
\begin{align}
\label{md-vector-result}i\bm{\omega}\FT{f}=\mathcal{F}(\nabla_{\bm{x}}f(\bm{x})),\\
\label{md-matrix-result}(i\bm{\omega})(i\bm{\omega})^\top\FT{f}=\FT{\nabla^{(2)}_{\bm{x}}f(\bm{x})},
\end{align}
where $\nabla^{(m)}_{\bm{x}}f$ denotes the $m$th order derivative of $f$ w.r.t.\ $\bm{x} \ \ \forall \ \ m\ge0$.

We formalise the $n$-dimensional Fourier transform in \Cref{multi-fourier}, and provide definitions for Fourier transforms of matrix and vector quantities. We also derive the differentiation/multiplication property in $n$-dimensional space.

\section{Main Result}
\label{sec-main-result}
 
In this section, we prove our main result. The motivating factor behind these derivations is that by viewing the inner integral $\hat{I}_\theta$ as a \emph{convolution}, we can analyse our policy gradient in the frequency domain. This affords powerful analytical results that enable us to exploit the multiplication/derivative property of Fourier transforms, namely manipulation of expressions involving the derivatives of our critic $\hat{Q}$ in the action-spatial domain are represented simply by factors of $(i\bm{\omega})$ in the frequency domain. We apply this elegant property in \Cref{freq} to demonstrate the relative ease of manipulation of the inner integral $\hat{I}_\theta$. In \Cref{estimators}, we show that existing Monte Carlo policy gradient estimators arise from our theorem as a single family of cases using different factors $(i\bm{\omega})$ multiplied with the critic $\hat{Q}$.

Moreover, our theorems rely only on the characteristic function $\mathcal{F}(\tilde{\beta})$. While the original technique developed for EPG \citep{epg-journal} relies on the moment generating function to obtain $\hat{I}$ for a policy from an exponential family and a polynomial family of critics, we require only that both the policy PDF and the critic have a closed form Fourier transform \cite{Karr1993}. For policies, this condition is easy to satisfy since almost all common distributions have a closed form characteristic function.

\begin{theorem}[Fourier Policy Gradients]
Let $\hat{I}_\theta(\bm{s}_t)= \nabla_{\theta}\int_{a}{\hat{Q}(\bm{s}_t,\bm{a})\beta_{\theta}(\bm{a}|\bm{s}_t)}d\bm{a}$ be the inner integral of the policy gradient for a critic $\hat{Q}(\bm{a})$ and policy $\beta(\bm{a})$ with auxiliary policy $\tilde{\beta}(\bm{\mu} - \bm{a}) = \beta(\bm{a})$. We may write $\hat{I_{\theta}}(\bm{s}_t)$ as:
\begin{align}
\label{eqmain}\hat{I_{\theta}}(\bm{s}_t)&=\nabla_{\theta}\mathcal{F}^{-1}\Big(\mathcal{F}(\hat{Q})\mathcal{F}(\tilde{\beta})\Big)(\bm{\mu}).
\end{align}
\end{theorem}
\begin{proof}
Recall the definition of $\hat{I_{\theta}}(\bm{s}_t)$ from \eqref{eq-iqm}: 
\begin{gather}
\label{eq-I_theta}
\hat{I_{\theta}}(\bm{s}_t)=\nabla_{\theta}\int_{a}\hat{Q}(\bm{s}_t,\bm{a}) \beta d\bm{a}.
\end{gather}
To exploit the convolution property of Fourier transforms given by \eqref{ft-convolution}, the first step is to introduce an auxiliary policy $\tilde{\beta}$, so that the above integral becomes a convolution. If the mean of the policy $\beta$ is $\bm{\mu}$, the new auxiliary policy $\tilde{\beta}$ is:
\begin{gather}
\tilde{\beta}(\bm{\mu} - \bm{a}) = \beta(\bm{a}). 
\end{gather}
We start by rewriting $\hat{I_{\theta}}$ as:
\begin{gather}
\label{pg-conv}
\hat{I_{\theta}}= \nabla_{\theta}\int_{a} \hat{Q}(\bm{a}) \tilde{\beta}(\bm{\mu} - \bm{a}) d\bm{a} = \nabla_{\theta}(\hat{Q} * \tilde{\beta}) (\bm{\mu}).
\end{gather}
Now, we apply the Fourier transform to the convolution and use \eqref{ft-convolution} to reduce it to a multiplication:
\begin{align}
\mathcal{F}(\hat{Q} * \tilde{\beta})(\bm{\omega}) &=\mathcal{F}\bigg(\int_{a}\hat{Q}(\bm{a})\tilde{\beta}(\bm{\mu}-\bm{a}))d\bm{a}\bigg)(\bm{\omega}), \\
&=\Big(\mathcal{F}(\hat{Q}) \mathcal{F}(\tilde{\beta})\Big)(\bm{\omega}). 
\end{align}
Taking the inverse Fourier transform gives:
\begin{align}
(\hat{Q} * \tilde{\beta})(\bm{\mu}) &= \mathcal{F}^{-1}\Big(\mathcal{F}(\hat{Q}) \mathcal{F}(\tilde{\beta})\Big)(\bm{\mu}).
\end{align}
Substituting this into \eqref{pg-conv} yields our main result:
\begin{align}
\hat{I}_{\theta}(\bm{s}) &= \nabla_{\theta}\mathcal{F}^{-1}\Big(\mathcal{F}(\hat{Q}) \mathcal{F}(\tilde{\beta})\Big)(\bm{\mu}).
\end{align}
\end{proof}

We now derive a variant of our main theorem for the special case of $\bm{\mu}$.

\begin{theorem}[Fourier Policy Gradients for $\bm{\mu}$]
\label{main1}
Let $\hat{I}_{\bm{\mu}}(\bm{s}_t)=\nabla_{\bm{\mu}}\int_{a}{\hat{Q}(\bm{s}_t,\bm{a})\beta_{\theta}(\bm{a}|\bm{s}_t)}d\bm{a}$ be the inner integral of the policy gradient for $\bm{\mu}$ with a critic $\hat{Q}(\bm{a})$ and policy $\beta(\bm{a})$ with auxiliary policy $\tilde{\beta}(\bm{\mu} - \bm{a}) = \beta(\bm{a})$. We may write $\hat{I}_{\bm{\mu}}(\bm{s}_t)$ as:
\begin{align}
\label{four-mu}\hat{I}_{\bm{\mu}}(\bm{s}_t)=\mathcal{F}^{-1}\Big(\mathcal{F}(\hat{Q})i\bm{\omega}\mathcal{F}(\nabla\tilde{\beta})\Big)(\bm{\mu}).
\end{align}
\end{theorem}

\begin{proof}
We return to \eqref{pg-conv}, retaining the derivative inside the integral: 
\begin{gather}
\hat{I}_{\bm{\mu}}(\bm{s})= \int_{a} \hat{Q}(\bm{a})\nabla_{\bm{\mu}} \tilde{\beta}(\bm{\mu} - \bm{a}) d\bm{a}.
\end{gather}
From the chain rule, we substitute $\nabla_{\bm{\mu}}\Big(\tilde{\beta}(\bm{\mu} - \bm{a})\Big)=\Big(\nabla\tilde{\beta}\Big)(\bm{\mu} - \bm{a})$, yielding:
\begin{gather}
\label{intermediate4}
\hat{I}_{\bm{\mu}}(\bm{s})=\int_{a} \hat{Q}(\bm{a})\Big(\nabla\tilde{\beta}\Big)(\bm{\mu} - \bm{a}) d\bm{a}=(\hat{Q}*\nabla\tilde{B})(\bm{\mu}).
\end{gather}
Now, we take the Fourier transforms of the convolution $(\hat{Q}*\nabla\tilde{B})(\bm{\mu})$ and exploit the multiplication property of \eqref{ft-convolution}:
\begin{align}
\label{intermediate2} 
\mathcal{F}((\hat{Q}*\nabla\tilde{B}))(\bm{\omega})&=\Big(\mathcal{F}(\hat{Q}) \mathcal{F}(\nabla\tilde{\beta})\Big)(\bm{\omega}).\\
\end{align}
Using the multiplication/derivative property from \eqref{md-vector-result}, we substitute for $\mathcal{F}(\nabla\tilde{\beta})=i\bm{\omega}\mathcal{F}(\tilde{\beta})$:
\begin{align}
\mathcal{F}((\hat{Q}*\nabla\tilde{B}))(\bm{\omega})&=\Big(\mathcal{F}(\hat{Q})i\bm{\omega}\mathcal{F}(\nabla\tilde{\beta})\Big)(\bm{\omega}).\\
\end{align}
Finally, taking inverse Fourier transforms and substituting into \eqref{intermediate4} yields our result:
\begin{gather}
\hat{I_{\theta}}(\bm{s})= \mathcal{F}^{-1}\Big(\mathcal{F}(\hat{Q}) \mathcal{F}(\nabla\tilde{\beta})\Big)(\bm{\mu}).
\end{gather}
\end{proof}
We use \Cref{main1} to derive the following corollary, valid for all parameters $\psi$ s.t.\ $\bm{\mu}$ does not depend upon them.

\begin{corollary}
Let $\psi$ be a parameter that does not depend upon $\bm{\mu}$. We can write $\hat{I}_\psi(\bm{s}_t)=\nabla_\psi\int_{a}{\hat{Q}(\bm{s}_t,\bm{a})\beta_{\theta}(\bm{a}|\bm{s}_t)}d\bm{a}$ as:
 \begin{gather}
 \label{eq-four}
 \hat{I}_{\psi}(\bm{s})=\mathcal{F}^{-1}\Big(\mathcal{F}(\hat{Q}) \nabla_{\psi}\mathcal{F}(\tilde{\beta})\Big)(\bm{\mu}).
 \end{gather}
\end{corollary}

The required auxiliary policy $\tilde{\beta}(\bm{a})=\beta(\bm{\mu}-\bm{a})$ exists for all distributions $\beta$ with unbounded support. For symmetric distributions, $\tilde{\beta}$ often has a  convenient form, e.g., for a Gaussian policy $\beta = \mathcal{N}(\bm{\mu}, \bm{\Sigma})$, $\tilde{\beta}=\mathcal{N}(\bm{0}, \bm{\Sigma})$. This transformation is similar to reparameterisation \citep{svg}. For critics, we discuss tractable critic families in the remainder of the paper.   

\section{Applications}
\label{sec-applications}
We now discuss a number of specialisations of \eqref{eqmain} and \eqref{four-mu}, linking them to several established policy gradient approaches. 

\subsection{Frequency Domain Analysis}
\label{freq}

We now motivate the remainder of this section by considering the Gaussian policy $\tilde{\beta} = \mathcal{N}(\bm{0}, \bm{\Sigma})$. We need to calculate the gradient w.r.t. $\bm{\Sigma}^{\frac{1}{2}}$ where $(\bm{\Sigma}^{\frac{1}{2}})^\top\bm{\Sigma}^{\frac{1}{2}}=\bm{\Sigma}$. From the characteristic function for a multivariate Gaussian, $\mathcal{F}(\mathcal{N}(\bm{0}, \bm{\Sigma}))=e^{-\frac{1}{2}\bm{\omega}^\top\bm{\Sigma}\bm{\omega}}$, we find derivatives as:
\begin{align}
\nabla_{\bm{\Sigma}^{\frac{1}{2}}}\mathcal{F}(\tilde{\beta})&=\nabla_{\bm{\Sigma}^{\frac{1}{2}}}e^{-\frac{1}{2}\bm{\omega}^\top\bm{\Sigma}\bm{\omega}},\\
&=-\frac{1}{2}\bm{\Sigma}^{\frac{1}{2}}\Big(2\bm{\omega}\bm{\omega}^\top e^{-\frac{1}{2}\bm{\omega}^\top\bm{\Sigma}\bm{\omega}}\Big),\\
&=\bm{\Sigma}^{\frac{1}{2}}(i\bm{\omega})(i\bm{\omega})^\top\mathcal{F}(\tilde{\beta}).
\end{align}
Substituting for $\nabla_{\psi=\bm{\Sigma}^{\frac{1}{2}}}\mathcal{F}(\tilde{\beta})$ in \eqref{eq-four} gives the gradient for $\bm{\Sigma}^{ \frac{1}{2} }$. For completeness, we also include the update for $\bm{\bm{\mu}}$ which, recall from \eqref{four-mu}, is the same for all policies with auxiliary function $\tilde{\beta}(\bm{a})=\beta(\bm{\mu}-\bm{a})$.
\begin{align}
\label{four-gauss-mu}\hat{I}_{\bm{\mu}} &= \mathcal{F}^{-1}\Big(\mathcal{F}(\hat{Q})(i\bm{\omega})\mathcal{F}(\tilde{\beta})\Big)(\bm{\mu}),\\
\label{four-gauss-sig}\hat{I}_{\bm{\Sigma}^{\frac{1}{2}}}&= \mathcal{F}^{-1}\Big(\bm{\Sigma}^{\frac{1}{2}}\mathcal{F}(\hat{Q})(i\bm{\omega})(i\bm{\omega})^\top\mathcal{F}(\tilde{\beta})\Big).
\end{align}

We see from \eqref{md-vector-result} and \eqref{md-matrix-result} that the terms $i \bm{\omega}$, once pulled into the Fourier transform, become differentiation operators. However, \eqref{four-gauss-mu} and \eqref{four-gauss-sig} afford us a choice -- we can pull them into the critic term or the policy term. This gives rise to a number of different expressions for the gradient. To differentiate between methods, we define the order of the method, denoted by $M$, the order of the derivative with respect to the critic.

We continue our example of Gaussian policies, using $\eqref{four-gauss-mu}$ to compute an update for $\bm{\mu}$ for $M\in\{0,1\}$ and  \eqref{four-gauss-sig} to compute an update for $\Sigma^{1/2}$ for $M\in\{0,1,2\}$. Full derivations with Gaussian derivatives can be found in \Cref{sec-gd}.

\paragraph{Zeroth Order Method ($\bm{M=0}$)}
Using \eqref{four-gauss-mu} and \eqref{four-gauss-sig} in their current form gives an analytic expression for a zeroth order critic, as we do not multiply $\mathcal{F}(\hat{Q})$ by any factor of $i\bm{\omega}$. Using results for multidimensional Fourier transforms from \eqref{md-vector-result} and \eqref{md-matrix-result} when taking these inverse transforms, we obtain:
\begin{align}
\hat{I}_{\bm{\mu}}&=\mathcal{F}^{-1}(\mathcal{F}(\hat{Q})\mathcal{F}(\nabla\tilde{\beta}))=\int_{a}\hat{Q}\nabla \tilde{\beta}d\bm{a} \\ &= -\int_{a}\hat{Q}\nabla \beta d\bm{a}, \label{mu-0} \\
\hat{I}_{\bm{\Sigma}^{\frac{1}{2}}}&=\mathcal{F}^{-1}(\bm{\Sigma}^{\frac{1}{2}}\mathcal{F}(\hat{Q})\mathcal{F}(\nabla^{(2)}\tilde{\beta}))=\bm{\Sigma}^{\frac{1}{2}}\int_{a}\hat{Q}\nabla^{(2)} \tilde{\beta}d\bm{a} \\ &=\bm{\Sigma}^{\frac{1}{2}} \int_{a}\hat{Q}\nabla^{(2)} \beta d\bm{a} \label{sig-0}.
\end{align}
Here, we use the identities $\nabla \tilde{\beta}=-\nabla \beta$ and $\nabla^{(2)} \tilde{\beta}=\nabla^{(2)} \beta$ from \Cref{nth-aux}.

\paragraph{First Order Method ($\bm{M=1}$)} To obtain an analytic expression in terms of $\nabla_{\bm{a}}\hat{Q}$, we must manipulate the factors of $(i\bm{\omega})$ in \eqref{four-gauss-mu} and \eqref{four-gauss-sig} to obtain a factor of $(i\bm{\omega})\hat{Q}$. We then exploit the multidimensional Fourier transform result for vectors from \eqref{md-vector-result} as before:
\begin{align}
\hat{I}_{\bm{\mu}} &= \IFT{ (i\bm{\omega}) \mathcal{F}(\hat{Q})\mathcal{F}(\tilde{\beta}) } = \IFT{\mathcal{F}(\nabla \hat{Q})\mathcal{F}(\tilde{\beta})}\\
&= \int_{a} \beta\nabla \hat{Q} d\bm{a} ,
\label{mu-1}\\
\hat{I}_{\bm{\Sigma}^{\frac{1}{2}}} &= \IFT{\bm{\Sigma}^{\frac{1}{2}}(i\bm{\omega})\mathcal{F}(\tilde{\beta})(i\bm{\omega})^\top\mathcal{F}(\hat{Q})}, \\
&= \IFT{\bm{\Sigma}^{\frac{1}{2}}\mathcal{F}(\nabla \tilde{\beta})\mathcal{F}(\nabla Q)^\top}, \\
&= -\bm{\Sigma}^{\frac{1}{2}}\int_{a}\nabla \beta(\nabla \hat{Q})^\top d\bm{a}. \label{sig-1}
\end{align}

\paragraph{Second Order Method ($\bm{M=2}$)}
We repeat the process, this time taking the derivative of $\hat{Q}$ twice:
\begin{align}
\hat{I}_{\bm{\Sigma}^{\frac{1}{2}}} &=\IFT{\bm{\Sigma}^{\frac{1}{2}}(i\bm{\omega})(i\bm{\omega})^\top\mathcal{F}(\hat{Q})\mathcal{F}(\tilde{\beta})},\\
&=\IFT{\bm{\Sigma}^{\frac{1}{2}}\mathcal{F}(\nabla^{(2)} \hat{Q})\mathcal{F}(\tilde{\beta})} =\bm{\Sigma}^{\frac{1}{2}}  \int_{a}\nabla^{(2)}_{\bm{a}}\hat{Q}\beta d\bm{a}. \label{sig-2}
\end{align}
Here, we exploit the multidimensional Fourier transform result for matrices from \eqref{md-matrix-result} in deriving the second line.

\subsection{Family of SPG Estimators}
\label{estimators}
We are going to revisit certain integrals from \Cref{freq} using the following rule for deriving Monte Carlo estimators:
\[
\int_{a} f(a) da \approx \frac1{\beta(a)} f(a) \quad \text{where} \quad a \sim \beta.
\]
Here, the quantity on the right is a sample-based approximation. We have the following approximations for the integrals given by equations (\ref{mu-0},\ref{sig-0},\ref{mu-1},\ref{sig-1},\ref{sig-2}), which recall were defined for a Gaussian policy $\beta=\mathcal{N}(\bm{\mu},\bm{\Sigma})$:
\begin{align}
&\begin{rcases}
\IP_{\bm{\mu}}&=-\int_{a}\hat{Q}\nabla_{\bm{a}}\beta d{\bm{a}} \approx\bm{\Sigma}^{-1}(\bm{a}_t-\bm{\mu})\hat{Q},\\
\IP_{\bm{\Sigma}^{\frac{1}{2}}}&=\bm{\Sigma}^{\frac{1}{2}}\int_{a}\hat{Q}\nabla^{(2)}_{\bm{a}}\beta d\bm{a},\\
 &\approx\Big((\bm{\Sigma}^{\frac{1}{2}})^{-\top}(\bm{a}_t-\bm{\mu})(\bm{a}_t-\bm{\mu})^\top\bm{\Sigma}^{-1}\\
&\quad-(\bm{\Sigma}^{\frac{1}{2}})^{-\top}\Big)\hat{Q},\\
\end{rcases}
\text{$M=0$}\\[10pt]
&\begin{rcases}
\IP_{\bm{\mu}}&=\int_{a}\nabla_{\bm{a}}\hat{Q}\beta d{\bm{a}} \approx\nabla_{\bm{a}}\hat{Q},\\
\IP_{\bm{\Sigma}^{\frac{1}{2}}}&=-\bm{\Sigma}^{\frac{1}{2}}\int_{a}\nabla_{\bm{a}}\beta(\nabla_{\bm{a}}\hat{Q})^\top d\bm{a},\\
&\approx(\bm{\Sigma}^{\frac{1}{2}})^{-\top}(\bm{a}_t-\bm{\mu})(\nabla_{\bm{a}}\hat{Q})^\top,
\end{rcases}
\text{$M=1$}\\[10pt]
&\begin{rcases}
\IP_{\bm{\Sigma}^{\frac{1}{2}}}&=\bm{\Sigma}^{\frac{1}{2}}\int_{a}\nabla^{(2)}_{\bm{a}}\hat{Q}\beta d\bm{a}\\
&\approx\bm{\Sigma}^{\frac{1}{2}}\nabla^{(2)}_{\bm{a}}\hat{Q}.
\end{rcases}
\text{$M=2$}
\end{align}

The above equations summarise existing results for stochastic policy gradients estimators, which are applicable for \textit{any} policy and critic. The zeroth-order results ($M=0$) correspond to standard policy gradient methods \citep{spg, reinforce, scorevi}; the first-order ones ($M=1$) correspond to reparameterisation-based methods \citep{svg, auto_bayes}; and, when applied to a Gaussian policy, the second-order ($M=2$) of the update for $\Sigma$ is a sample-based version of \emph{Gaussian policy gradients} \citep{epg}, a special case of EPG. Note that interpolations between different estimators can also be used \citep{qprop} as a method of reducing variance further. The full derivations for of the derivatives for the multivariate Gaussian are given in Appendix \ref{sec-gd}.

\subsection{Periodic Action Spaces}
\label{sec-periodic}
In some settings, the action space of an MDP is naturally periodic, e.g., when actions specify angles. By using a trigonometric function in the critic, we encode the insight that rotating by $-\pi$ and by $\pi$ leads to similar results, despite the fact that the two points lie on the opposite ends of the action range.

Consider the case where the policy is Gaussian, i.e., $\beta = \mathcal{N}(\BMU, \BSIGMA)$ and the critic $\hat{Q}$ is a trigonometric function of the form 
\[
\hat{Q}(a) = \cos (\BF ^\top \BA - h),
\]
where $\BF \in \mathbb{R}^n, h \in \mathbb{R}$, and $n$ is the dimension of the action space. 

While a policy gradient method involving a critic $\hat{Q}$ of this form superficially resembles approximating the value function with the Fourier basis \citep{fourier_value} for the state space, it is in fact completely different. Indeed, our method uses a Fourier basis to approximate a function of the \emph{action space}, not the \emph{state space}, which often has different structure. The dependence of $\hat{Q}$ on the state can still be completely arbitrary (for example a neural network). 

We seek to find the policy gradient update for this combination of critic and policy. First, we write out their Fourier transforms:
\begin{align}
\mathcal{F}(\hat{Q}) &= (2\pi)^n\left[\frac{e^{-ih}\delta(\BO - \BF) + e^{ih}\delta(\BO + \BF)}{2} \right] , \\
\mathcal{F}(\tilde{\beta}) &= e^{-\frac12 \BO^\top \Sigma \BO}.
\end{align}
Computing the inverse Fourier transform yields:
\begin{gather}
\label{eq-periodic-j}
\mathcal{F}^{-1}(\mathcal{F}(\hat{Q}) \mathcal{F}(\tilde{\beta}))(\bm{a}) = e^{-\frac12 \BF^\top \BSIGMA \BF} \cos(\BF^\top \BA - h).
\end{gather}
A more detailed derivation of \eqref{eq-periodic-j} can be found in \cref{periodic_critic_derivation}. We now use \eqref{eqmain} to obtain the policy gradients for the mean and the covariance.
\begin{align}
\IP_{\BMU} &= \nabla_{\BMU} \mathcal{F}^{-1}(\mathcal{F}(\hat{Q}) \mathcal{F}(\tilde{\beta}))(\BMU), \\ &= - e^{-\frac12 \BF^\top \BSIGMA \BF} \sin(\BF^\top \BMU - h) \BF, \label{fepg-mu} \\
\IP_{ \BSIGMA } &= \nabla_{\BSIGMA} \mathcal{F}^{-1}(\mathcal{F}(\hat{Q}) \mathcal{F}(\tilde{\beta}))(\BMU), \\ &= - e^{-\frac12 \BF^\top \BSIGMA \BF}  \frac12 \BF \BF^\top \cos(\BF^\top \BMU - h). \\
\end{align}
Intuitively, the mean update contains a frequency damping component $e^{-\frac12 \BF^\top \BSIGMA \BF} $, which is small for large $\BF$, ensuring that the optimisation slows down when the signal is repeating frequently. The covariance update uses the same damping, while also making sure that exploration increases in the minima of the critic and decreases near the maxima, in a way slightly similar, but mathematically different, from Gaussian policy gradients \citep{epg, epg-journal}.

We evaluated a periodic critic of this form on a toy \emph{turntable domain} where the goal is to rotate a flat record to the desired position by rotating it (see Appendix \ref{turntable-setup} for details). We compared it to the DPG baseline from OpenAI \cite{baselines}, which uses a neural network based critic capable of addressing complex control tasks. As expected, the learning curves in Figure \ref{fig-tte} show that using a periodic critic (F-EPG) leads to faster learning, because  it encodes more information about the action space than a generic neural network. Our method efficiently uses this information in the policy gradient framework by deriving an exact policy gradient update. 

\begin{figure}[tb]
\hspace{-1cm} \includegraphics[width=9cm]{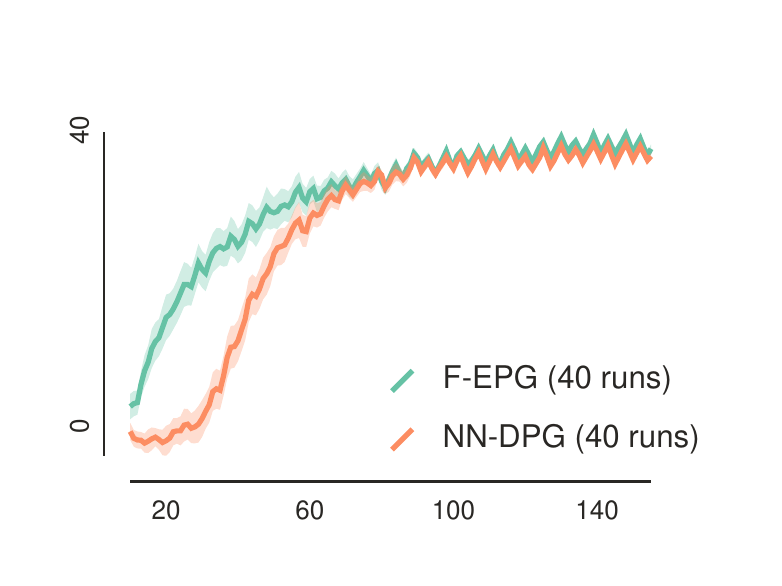}
\centering
\caption{Learning curves for Turntable. EPG with periodic critic (F-EPG) vs.\ DPG with a neural network critic (NN-DPG).}
\label{fig-tte}
\end{figure}

\subsection{Policy Gradients with Radial Basis Functions}
\label{sec-rbf}
\emph{Radial basis functions} (RBFs) \citep{buhmann2003radial} have a long tradition as function approximators in machine learning and combine a simple, tractable structure with the universal approximation property \cite{park1991universal}. In this section, we analyse the elementary RBF building block -- a \emph{single} RBF node.  Results on combining many such blocks are deferred to \Cref{ss-hc}.

Consider the setting where the policy is Gaussian, i.e., $\beta = \mathcal{N}(\BMU, \BSIGMA)$ and the critic is an RBF $\hat{Q} = \mathcal{N}(\BL, \BS)$. Although the critic $\hat{Q}$ has the shape of a Gaussian PDF, it is not a random variable but simply an ordinary function parametrised by the location vector $\BL$ and the positive-definite scale matrix $\BS$, which occupy the place of the mean and the covariance. We want to find the policy gradient updates for the mean and the covariance. We begin the derivation by writing out the Fourier transforms for the policy and the critic:
\begin{align}
\mathcal{F}(\hat{Q}) &= e^{i \BL^\top \BO - \frac12 \BO^\top \BS \BO},\\
\mathcal{F}(\tilde{\beta}) &= e^{- \frac12 \BO^\top \BSIGMA \BO}. 
\end{align}
The inverse Fourier transform has the following form:
\begin{align}
\mathcal{F}^{-1}(\mathcal{F}(\hat{Q}) \mathcal{F}(\tilde{\beta}))(\bm{a}) = \mathcal{N}(\BL, \BSIGMA + \BS)(\bm{a}).
\end{align}
Now, we substitute $\bm{a} = \bm{\mu}$ and introduce the notation:
\begin{gather}
E = \mathcal{N}(\BL, \BSIGMA + \BS)(\BMU).
\end{gather}
We now derive the policy gradients using \eqref{eqmain} and properties of the derivative of the logarithm.
\begin{align}
\IP_\BMU &= \nabla_\BMU E = E \nabla_\BMU \log E, \\ &= - \frac 12 E \nabla_\BMU \| \BMU - \BL \|^2_{(\BSIGMA + \BS)^{-1}}, \label{pg-rbf-mu} \\
\IP_\BSIGMA &= \nabla_\BSIGMA E = E \nabla_\BSIGMA \log E, \\ &= - \frac 12 E \left( \nabla_\BSIGMA \| \BMU - \BL \|^2_{(\BSIGMA + \BS)^{-1}} + \nabla_\BSIGMA \log \det(2 \pi \BSIGMA) \right).
\end{align}
The RBF Policy gradient simply minimises the Mahalanobis distance with the weight matrix $(\BSIGMA + \BS)^{-1}$. Also, since $E$ is a positive scalar, for multi-dimensional action spaces, the multiplication by $E$ in the gradients does not change the gradient direction, only the magnitude.

For the mean update $\nabla_\BMU$, this result is intuitive -- if we want our policy to reach the maximum of the RBF node (i.e., a bump) we simply minimise the distance between the current policy mean and the top of the bump. We now provide an additional variant of this result, based on natural policy gradients. The Fisher matrix for the Gaussian distribution parameterised by $\BMU$ (with the covariance kept constant) is simply $\BSIGMA$, yielding the following update:
\begin{align}
g_{\text{natural}} &= - \frac12 (\BSIGMA) (\BSIGMA + \BS)^{-1} (\BMU - \BL), \\
&= - \frac12 (\BS \BSIGMA^{-1} + \text{Id})^{-1} (\BMU - \BL). 
\end{align}
Here, the symbol $\text{Id}$ denotes the identity matrix. The update given by $g_{\text{natural}}$ can be used in place of $I_\BMU$ to obtain a natural policy gradient method. Moving from a standard first order policy gradient to the natural policy gradient is simply a change in the weighting matrix of the Mahalanobis distance from  $(\BSIGMA + \BS)^{-1}$ to $(\BS \BSIGMA^{-1} + \text{Id})^{-1}$. Furthermore, the Mahalanobis distance reduces to the unweighted $L_2$ distance when $\BS = \BSIGMA$. Intuitively, since the natural policy gradient takes the geometry of the space of distributions into account, a simpler update is obtained if this geometry is the same as the geometry of the RBF (as given by $\BS$). 

\subsection{Revisiting Gaussian Policy Gradients }
\label{sec-gpg}
In this section, we revisit Gaussian policy gradients \citep{epg} with the aim of contrasting it with the RBF derivation presented above. Gaussian policy gradients assume that the policy is Gaussian, i.e., $\beta = \mathcal{N}(\BMU, \BSIGMA)$, and the critic is quadric, i.e.,
\begin{gather}
\hat{Q}(a) = \bm{a}^\top \bm{H} \bm{a} + \bm{a}^\top \bm{b} + \text{const} = (\bm{a} - \BL)^\top \bm{H} (\bm{a} - \BL).
\end{gather}  
Here we denote by `const' some constant which always exists so that the above equality holds for $\BL = -\frac12 \bm{H}^{-1} \bm{b}$. 

In this setting we have that,
\begin{align}
E & = \E_{\beta}\left[ (\bm{a}-\BL)^\top \bm{H} (\bm{a}-\BL) \right], \\ 
& = \text{trace}(\bm{H}\bm{\Sigma}) + (\BMU-\BL)^\top \bm{H} (\BMU-\BL).
\end{align}
Now, we compute the policy gradient for the mean:
\begin{gather}
\IP_\BMU = \nabla_\BMU E = \nabla_\BMU (\BMU-\BL)^\top \bm{H} (\BMU-\BL).
\end{gather}
This is almost the same as \eqref{pg-rbf-mu}, except for a positive scaling factor and the fact that, $\bm{H}$ need not be positive-definite (unlike the matrix $\BS$ from the definition of the RBF.). This illustrates both the similarities and differences between Gaussian policy gradients, which uses quadrics, and the updates given by \eqref{pg-rbf-mu}, which is based on RBFs. The similarity is that we are minimising a quadratic form, while the difference is that the quadratic form used by RBF comes from a more restrictive family (i.e., it has to be positive definite). However, RBFs also have some advantages over quadrics in that they are bounded both from above and below.    

\subsection{Hybrid Critics}
\label{ss-hc}
We now consider the case when the critic $\hat{Q}$ is a linear combination, i.e., $\hat{Q}(\bm{a}) = \sum_i c_i \hat{Q}_i(\bm{a})$ for some $c_i$. The main observation is that the integral $\hat{I}$ is linear in the critic, i.e., for any parameter $\theta$ we have that, 
\begin{gather}
\IP_\theta = \sum_i c_i \{ \IP_\theta \}_i + (\nabla_\theta c_i) E_i, \quad \text{where} \\ E_i = \int_{\bm{a}} \hat{Q}_i \beta(\bm{a}) d \bm{a} \quad \text{and} \quad \{ \IP_\theta \}_i = \nabla_\theta E_i.
\end{gather}

Thus, we can compute the policy gradient update for each component of the critic separately, and then use a linear combination of the updates. If each of these components is in a tractable family, such as a trigonometric function (Section \ref{sec-periodic}), an RBF (Section \ref{sec-rbf}) or a polynomial \citep{epg-journal}, the whole update is also tractable. In this way, we can use critics consisting not just of of a superposition of functions from a single family (like a Fourier basis, which consists of different trigonometric functions), but also hybrid ones, combining functions from many families. 

All these three categories of critics have their corresponding universal approximation result, implying that a linear combination of a sufficient number of functions from that class alone is rich enough to approximate any reasonable function on a bounded interval to arbitrary accuracy. Indeed, we have the Weierstrass theorem about linear combinations of monomials \citep{weierstrass1885analytische,stone1948generalized}, the result by \citet{park1991universal} for linear combinations of RBFs, and the Fourier series approximation for linear combinations of trigonometric functions (see Appendix \ref{fseries}). 

These results show that, in principle, we can have analytic updates for a critic matching \emph{any} $Q$-function, and hence any MDP, without the need for Monte Carlo sampling schemes similar to \eqref{spg-mc}, with no sampling noise (given the state) and with virtually no computational overhead relative to stochastic policy gradients. However, there remain two obstacles. First, for a finite number of basis functions, the approximation may introduce spurious local minima that are harmful to any local optimisation method. Second, even when local minima are not a problem, there is a case for using a degree of sampling in case we believe that our critic $\hat{Q}$ is biased -- some sampling methods allow the use of direct reward rollouts to address bias. We believe that the practical impact of our analytic results and the question of which critic combination to use is yet to be determined.       

\subsection{Mixture Policies and Other Nonstandard Policies}
\label{sec-pi-char}
We now consider mixture policies of the form $\beta(\bm{a}) = \sum_i b'_i \beta_i(\bm{a})$, where  $b'_i \geq 0$ and $\sum_i b'_i = 1$. Similarly to the previous section, we use the linearity of the integral:
\begin{gather}
\label{eq-mixture-pi}
\IP_\theta = \sum_i \{ b'_i \IP_\theta \}_{\beta_i} + (\nabla_\theta b'_i)
E_{\beta_i}, \quad \text{where} \\ E_{\beta_i} = \int_{\bm{a}} \hat{Q}(\bm{a}) \beta_i(\bm{a}) d\bm{a}  
\quad \text{and} \quad \{ \IP_\theta \}_{\beta_i} = \nabla_\theta E_{\beta_i}.
\end{gather}
The two most common types of components for the policy are Gaussian and the deterministic policy (i.e., a Dirac-delta measure). Hence, using \eqref{eq-mixture-pi}, we can obtain a policy gradient method for policies that have several modes (modelled with Gaussians) as well as several focussed (discrete) points. Of course, we can also use any other distribution with a characteristic function by substituting into \eqref{eqmain}. We believe that such policies can be particularly useful in multi-agent settings, where the concept of finding a maximum of the total expected return generalises to finding a Nash equilibrium and it is known that some Nash equilibria admit only stochastic policies \citep{nisan2007algorithmic}. It is also possible to have \emph{both} a mixture policy and a hybrid critic. We do not give the formula, since it is straightforward to derive. 

\section {Conclusions}
This paper developed new theoretical tools for deriving policy gradient updates, showing that expected policy gradients are tractable in three important classes of critics and for almost all policies. We also discussed a framework for deriving estimators for stochastic policy gradients, which generalises existing approaches. Moreover, we addressed the setting of MDPs with periodic action spaces and described an experiment demonstrating the benefits of explicitly modelling periodicity in a policy gradient method.

\section* {Acknowledgements}
This project has received funding from the European Research Council (ERC) under the European Union's Horizon 2020 research and innovation programme (grant agreement number 637713), and the Engineering and Physical Sciences Research Council (EPSRC).

\bibliography{FPG.bib}
\bibliographystyle{icml2018}
 
\newpage
\appendix
\section{Fourier Series and Approximation}
\label{fseries}
Formally, the Fourier series is an expansion of a periodic function $f(x)$ of period $2L$ in terms of an infinite summation of sines and cosines. For clarity, we give the univariate case -- the multivariate result can be found in literature. 
\[
\label{realfourier}f(x)=u_0+\sum_{m=1}^{\infty}u_m\cos(m\omega_0x)+\sum_{m=1}^{\infty}v_m\sin(m\omega_0x),\]
where $\omega_0\triangleq\frac{\pi}{L}$ and the coefficients for the series are:
\begin{align}
u_m&=\frac{1}{L}\int^L_{-L}f(x)\cos(m\omega_0x)dx,\\
v_m&=\frac{1}{L}\int^L_{-L}f(x)\sin(m\omega_0x)dx.
\end{align}
By writing sine and cosine terms in their complex exponential forms, it is possible to define a complex Fourier series for real valued functions as
\begin{align}
\label{complexfourier} f(x)&=\sum_{m=-\infty}^{m=\infty}c_me^{im\omega_0x},\\
c_m&=\frac{1}{2L}\int^L_{-L}f(x')e^{-im\omega_0x'}dx'.
\end{align}
\eqref{realfourier} and \eqref{complexfourier} are equivalent if we set $c_m$ as:
\[
c_m=\begin{cases}
\frac{1}{2}(u_m+iv_m) & \text{for }m<0,\\
u_0 & \text{for }m=0,\\
\frac{1}{2}(u_m-iv_m) & \text{for }m>0.
\end{cases}
\]
In reality, we cannot sum to infinity and instead use the series to approximate $f(x)$ to a finite value of $m$. Just as a Taylor series approximation becomes more accurate by using higher and higher order polynomials $x^m$, a Fourier series expansion becomes more accurate by using sinusoids of higher and higher frequencies $m\omega_0$. However, a Fourier series approximation approximates the function over its whole period, whereas the Taylor series does so only in a local neighbourhood of the given point.

Although the Fourier series is defined for periodic functions, it is still applicable to aperiodic functions. For bounded aperiodic functions, we define the period $2L$ to be the size of the domain of $f(x)$ and then integrate over this domain to obtain the Fourier coefficients. Intuitively, this is equivalent to repeating the bounded function periodically over an infinite domain. Aperiodic functions that are not bounded may be approximated by defining Fourier series over a bounded region of the function. As the size of this bounded region increases, and consequently the period $2L$ increases, the Fourier series approximation becomes more accurate and approaches a Fourier transform. Thus, for aperiodic unbounded functions, a Fourier series approximates a Fourier transform. 

We now formalise the idea of taking the limit of the period going to infinity ($L\rightarrow\infty$) for a complex Fourier series representation of any general function $f(x)$. Firstly, it is convenient to rewrite \eqref{complexfourier} as:
\[
f(x)=\frac{1}{2\pi}\sum_{m=-\infty}^{m=\infty}\int^L_{-L}f(x')e^{-im\omega_0x'}dx'e^{im\omega_0x}\omega_0.
\]
Taking the limit as $L\rightarrow\infty$ \citep{fourieranalysis} gives 
\[
f(x)=\underbrace{\frac{1}{2\pi}\int_{\omega}\bigg( \overbrace{\int_{x'} f(x')e^{-i\omega x'}dx'}^{\FT{f}}\bigg)e^{i\omega x}d\omega}_{\IFT{ \FT{f}}},
\]
which is exactly equivalent to \eqref{eq-ft-inv-ft}.

The integrals in the definition of the Fourier transform arise from taking a Fourier series representation of a function and letting the number of coefficients go to infinity.  

\section{\label{multi-fourier} $\bm{n}$-Dimensional Fourier Transforms}
\paragraph{Definitions}
Firstly, we make the definition of a $n$-dimensional Fourier transform precise: Consider a function $f(\cdot):\mathbb{R}^n\rightarrow\mathbb{R}$. For $\bm{x}=(x_1,x_2,...x_n)^\top\in\mathbb{R}^n$ and $\bm{\omega}=(\omega_1,\omega_2,...\omega_n)^\top\in\mathbb{R}^n$, we have:
\begin{align}
\FT{f(x)}\triangleq&\int_{x} f(\bm{x})e^{-i\bm{\omega}^\top\bm{x}}d\bm{x},\\
&=\overbrace{\int_{x_1}...\int_{x_n}}^{n}f(\bm{x})e^{-i\bm{\omega}^\top\bm{x}}dx_1...dx_n.
\end{align}
The corresponding $n$-Dimensional  inverse Fourier transform is defined as:
\begin{align}
\IFT{f(x)}\triangleq&\Big(\frac{1}{2\pi}\Big)^n\int_{\omega} f(\bm{x})e^{i\bm{\omega}^\top\bm{x}}d\bm{\omega},\\
&=\Big(\frac{1}{2\pi}\Big)^n\overbrace{\int_{\omega_1}...\int_{\omega_n}}^{n}f(\bm{x})e^{i\bm{\omega}^\top\bm{\omega}}d\omega_1...d\omega_n.
\end{align}
We define the Fourier transform of a vector/matrix quantity as simply the Fourier transform of individual elements of the vector/matrix. For example, the Fourier transform of matrix  $\big[\bm{F}(\bm{x})\big]_{jk}=f_{jk}(\bm{x})$ is found from:
\begin{align}
\label{vectormatrix-four}\Big[\mathcal{F}(\bm{F}(\bm{x}))\Big]_{jk}\triangleq\FT{f_{jk}(\bm{x})}.
\end{align}
And similarly for the inverse Fourier transform:
\begin{align}
\Big[\mathcal{F}^{-1}(\bm{F}(\bm{x}))\Big]_{jk}\triangleq\IFT{f_{jk}(\bm{x})}.
\end{align}

\paragraph{Multiplication-Derivative Identities}
We now derive multi-dimension analogues to the single dimension multiplication-derivative property, which we state here:
\begin{gather}
\label{eq-ft-der} \FT{\frac{\partial}{\partial x_j}f(\bm{x})}=i\omega_j\mathcal{F}(f(\bm{x})).
\end{gather}
Proofs of \eqref{eq-ft-der} are commonplace in Fourier Analysis references \cite{fourieranalysis}.
We start with a vector identity:

\begin{lemma}[Multiplication-Derivative Property: Vectors]
\label{md-vector}
Given a function $f(\bm{x})$ with Fourier transform $\FT{f(\bm{x})}$, multiplying $\FT{f(\bm{x})}$ by the vector $i\bm{\omega}$ in the frequency domain is equivalent to taking the first order derivative $\nabla_{\bm{x}}f(\bm{x})$ in the action domain, that is:
\begin{gather}
i\bm{\omega}\FT{f(\bm{x})}=\mathcal{F}(\nabla_{\bm{x}}f(\bm{x})).
\end{gather}
\end{lemma}
 
\begin{proof}
Consider the elements of the vector $i\bm{\omega}\FT{f(\bm{x})}$:
\begin{gather}
\label{element-vector}\Big[i\bm{\omega}\FT{f(\bm{x})}\Big]_j=i\omega_j\FT{f(\bm{x})}.
\end{gather}
Using the single dimension multiplication-derivative property from \eqref{eq-ft-der} yields:
\begin{gather}
\Big[i\bm{\omega}\FT{f(\bm{x})}\Big]_j=\FT{\frac{\partial}{\partial x_j}f(\bm{x})}.
\end{gather}
Using the definition of the Fourier transform of a vector from \eqref{vectormatrix-four} gives our main result:
\begin{gather}
i\bm{\omega}\FT{f(\bm{x})}=\mathcal{F}(\nabla_{\bm{x}}f(\bm{x})).
\end{gather}
\end{proof}

We now derive a similar identity for matrices:
\begin{lemma}[Multiplication-Derivative Property: Matrices]
\label{md-matrix}
Given a function $f(\bm{x})$ with Fourier transform $\FT{f(\bm{x})}$, multiplying $\FT{f(\bm{x})}$ by the matrix $(i\bm{\omega})(i\bm{\omega})^\top$ in the frequency domain is equivalent to taking the second order derivative $\nabla^{(2)}_{\bm{x}}f(\bm{x})$ in the action domain, that is:
\begin{gather}
(i\bm{\omega})(i\bm{\omega})^\top\FT{f(\bm{x})}=\FT{\nabla^{(2)}_{\bm{x}}f(\bm{x})}.
\end{gather}
\end{lemma}
 
\begin{proof}
Consider the elements of the matrix $(i\bm{\omega})(i\bm{\omega})^\top\FT{f(\bm{x})}$:
\begin{gather}
\label{element-vector}\Big[(i\bm{\omega})(i\bm{\omega})^\top\FT{f(\bm{x})}\Big]_{jk}=(i\omega_j)(i\omega_k)\FT{f(\bm{x})}.
\end{gather}
Using the single dimension multiplication-derivative property from \eqref{eq-ft-der} twice yields:
\begin{align}
\Big[(i\bm{\omega})(i\bm{\omega})^\top\FT{f(\bm{x})}\Big]_{jk}&=(i\omega_j)\FT{\frac{\partial}{\partial x_k}f(\bm{x})},\\
&=\FT{\frac{\partial^2}{\partial x_j\partial x_k}f(\bm{x})}.
\end{align}
Using the definition of the Fourier transform of a matrix from \eqref{vectormatrix-four} gives our main result:
\begin{gather}
i\bm{\omega}\FT{f(\bm{x})}=\mathcal{F}({\nabla^{(2)}_{\bm{x}}f(\bm{x})}).
\end{gather}
\end{proof}

\section{\label{auxiliary} Auxiliary Function Properties}
\begin{lemma}[$n$th Order Derivative of Auxiliary Function]
\label{nth-aux}
Given an auxiliary function $\tilde{\beta}(\bm{\mu}-\bm{a})=\beta(\bm{a})$ for a policy $\beta$, we may relate the $m$-th order derivative of $\tilde{\beta}$ w.r.t. $\bm{\mu}$ to the $m$th order derivative of $\beta$ w.r.t. $\bm{a}$ as:
\begin{gather}
\Big(\nabla^{(m)}\tilde{\beta}\Big)(\bm{\mu}-\bm{a})=(-1)^n\nabla^{(m)}_{\bm{a}}\beta(\bm{a})\ \ \forall \ \ m\ge0.
\end{gather}
\end{lemma}
 
\begin{proof}
\textit{For $m=1$}
From the chain rule we write:
\begin{align}
\Big(\nabla\tilde{\beta}\Big)(\bm{\mu}-\bm{a})=\nabla_{\bm{\mu}}\tilde{\beta}(\bm{\mu}-\bm{a}).
\end{align}
Let $\bm{\nu}=\bm{\mu}-\bm{a}$ s.t. $\tilde{\beta}(\bm{\mu}-\bm{a})=\tilde{\beta}(\bm{\nu})$. Using the chain rule again for $\nabla_{\bm{\mu}}\tilde{\beta}(\bm{\mu}-\bm{a})$ yields:
\begin{align}
\nabla_{\bm{\mu}}\tilde{\beta}(\bm{\nu})=\nabla_{\bm{\mu}}\bm{\nu}\nabla_{\bm{\nu}}\bm{a}\nabla_{\bm{a}}\tilde{\beta}(\bm{\nu}).
\end{align}
 Now, $\nabla_{\bm{\mu}}\bm{\nu}=\bm{I}$ and $\nabla_{\bm{\nu}}\bm{a}=-\bm{I}$. Substituting yields:
 \begin{align}
\nabla_{\bm{\mu}}\tilde{\beta}(\bm{\mu}-\bm{a})=(-1)\nabla_{\bm{a}}\tilde{\beta}(\bm{\nu}).
\end{align}
 Substituting $\tilde{\beta}(\bm{\nu})=\tilde{\beta}(\bm{\mu}-\bm{a})=\beta(\bm{a})$ gives our main result for $m=1$:
 \begin{align}
\Big(\nabla\tilde{\beta}\Big)(\bm{\mu}-\bm{a})=(-1)\nabla_{\bm{a}}\beta(\bm{a}).
\end{align}
Finally, taking $m-1$ more derivatives will give our main result:
 \begin{align}
\Big(\nabla^{(m)}\tilde{\beta}\Big)(\bm{\mu}-\bm{a})=(-1)^m\nabla^{(m)}_{\bm{a}}\beta(\bm{a}).
\end{align}
\end{proof}

\section{Turntable Experimental Setup Details}
\label{turntable-setup} 
The turntable domain is a toy continuous control task. The goal is to align a disk to a desired angle by rotating it around its axis. The action is an angle in the range $a \in [ -\pi, \pi ]$ and the observations are the current position of the disk and the target position, both expressed as angles. The reward is set to $ \sin (\alpha + \alpha_{\text{target}}) - \frac14 |a| $. For DPG, we used the OpenAI baseline implementation, where both the actor and the critic are represented using neural networks. For Fourier-EPG, we used the same setup but changed the critic to be trigonometric critic of the form $ \sin(\alpha + \alpha_{\text{target}} - a) + w |a|$ with a tuneable weight $w$ and the actor update given by Equation \eqref{fepg-mu}. The exploration policy was Gaussian with fixed standard deviation $\sigma=0.05$ in both cases.

\section{Gaussian Derivatives}
We derive specific analytical solutions for the Gaussian policy $\beta=\mathcal{N}(\bm{\mu},\bm{\Sigma})$ from \Cref{freq}. The following identities \citep{matrix_cookbook} will be useful:
\begin{align}
\label{d1-gauss}\nabla_{\bm{a}}\beta&=-\bm{\Sigma}^{-1}(\bm{a}-\bm{\mu})\beta,\\
\label{d2-gauss}\nabla^{(2)}_{\bm{a}}\beta&=\Big(\bm{\Sigma}^{-1}(\bm{a}-\bm{\mu})(\bm{a}-\bm{\mu})^\top\bm{\Sigma}^{-1}-\bm{\Sigma}^{-1}\Big)\beta.\\
\end{align}

\paragraph{Zeroth order $\bm{(M=0)}$}
Substituting for $\nabla_{\bm{a}}\beta$ from \eqref{d1-gauss} and $\nabla^{(2)}_{\bm{a}}\beta$ from \eqref{d2-gauss} in \eqref{mu-0} and \eqref{sig-0} respectively, we obtain our analytic expression:
\begin{align}
\hat{I}_{\bm{\mu}}=&\int_a\bm{\Sigma}^{-1}(\bm{a}-\bm{\mu})\hat{Q}\beta d\bm{a},\\
\hat{I}_{\bm{\Sigma}^{\frac{1}{2}}}=&\int_a\Big((\bm{\Sigma}^{\frac{1}{2}})^{-\top}(\bm{a}-\bm{\mu})(\bm{a}-\bm{\mu})^\top\bm{\Sigma}^{-1}\\
&-(\bm{\Sigma}^{\frac{1}{2}})^{-\top}\Big)\hat{Q}\beta d\bm{a}.
\end{align}

\paragraph{First order $\bm{(M=1)}$}
Substituting for $\nabla_{\bm{a}}\beta$ from \eqref{d1-gauss} in \eqref{sig-1}, we obtain our analytic expression:
\begin{align}
\hat{I}_{\bm{\Sigma}^{\frac{1}{2}}}&=\int_a(\bm{\Sigma}^{\frac{1}{2}})^{-\top}(\bm{a}-\bm{\mu})(\nabla_{\bm{a}}\hat{Q})^\top\beta d\bm{a}.
\end{align}
\label{sec-gd}
 
\section{Proofs}
\begin{customcorollary}{2.1}
Let $\psi$ be a parameter that does not depend upon $\bm{\mu}$. We can write $\hat{I}_\psi(\bm{s}_t)=\nabla_\psi\int_{a}{\hat{Q}(\bm{s}_t,\bm{a})\beta_{\theta}(\bm{a}|\bm{s}_t)}d\bm{a}$ as:
 \begin{gather}
 \hat{I}_{\psi}(\bm{s})=\mathcal{F}^{-1}\Big(\mathcal{F}(\hat{Q}) \nabla_{\psi}\mathcal{F}(\tilde{\beta})\Big)(\bm{\mu}).
 \end{gather}
\end{customcorollary}
\begin{proof}
Using \Cref{main1}, we obtain the following expression for $\hat{I}_\psi(\bm{s}_t)$:
\begin{gather}
\hat{I_\psi}(\bm{s}_t)=\nabla_\psi\mathcal{F}^{-1}\Big(\mathcal{F}(\hat{Q})\mathcal{F}(\tilde{\beta})\Big)(\bm{\mu}).
\end{gather}
Using Leibniz's rule for integration under the integral, we move the derivative inside of the inverse Fourier transform, obtaining our result:
\begin{gather}
\hat{I_{\psi}}(\bm{s}_t)=\mathcal{F}^{-1}\Big(\mathcal{F}(\hat{Q})\nabla_{\psi}\mathcal{F}(\tilde{\beta})\Big)(\bm{\mu}).
\end{gather}
\end{proof}

\section{Complete Periodic Critic Derivation}
\label{periodic_critic_derivation}
We now derive the analytic update from \eqref{eq-periodic-j} for our periodic critic. Firstly, for ease of analysis we re-write our critic using the hyperbolic function:
\begin{align}
\hat{Q}(a)=&\cos(\BF^\top\BA-h),\\
=&\frac{e^{i(\BF^\top\BA-h)}+e^{-i(\BF^\top\BA-h)}}{2},\\
=&\frac{e^{-ih}e^{i\BF^\top\BA}+e^{ih}e^{-i\BF^\top\BA}}{2}.
\end{align}
Taking the Fourier transform yields:
\begin{align}
\FT{\hat{Q}}=&\frac12\bigg[e^{-ih}(2\pi)^n\prod_{j=1}^{n}\delta(\omega_j-f_j)\\
&+e^{ih}(2\pi)^n\prod_{j=1}^{n}\delta(\omega_j+f_j)\bigg],\\
=&(2\pi)^n\left[\frac{e^{-ih}\delta(\BO-\BF)+e^{ih}\delta(\BO+\BF)}{2}\right].
\end{align}
Recall that the characteristic function of the Gaussian auxiliary function is $\FT{\tilde{\beta}}=e^{-\BO^\top\BSIGMA\BO}$. Now taking inverse Fourier transforms of $\mathcal{F}(\hat{Q}) \mathcal{F}(\tilde{\beta})$ yields:
 \begin{gather}
\mathcal{F}^{-1}(\mathcal{F}(\hat{Q}) \mathcal{F}(\tilde{\beta}))(\bm{a})=\frac{1}{(2\pi)^n}\int\mathcal{F}(\hat{Q}) \mathcal{F}(\tilde{\beta})e^{i\BO^T\BA}d\BO,\\
=\frac12\int e^{-\BO^\top\BSIGMA\BO} \left[e^{-ih}\delta(\BO-\BF)+e^{ih}\delta(\BO+\BF)e^{i\BO^T\BA}\right]d\BO,\\
=\frac12\int e^{-\BO^\top\BSIGMA\BO}\left[e^{i(\BO^T\BA-h)}\delta(\BO-\BF)+e^{i(\BO^T\BA+h)}\delta(\BO+\BF)\right]d\BO,\\
=e^{-\BF^\top\BSIGMA\BF}\left[\frac{e^{i(\BF^T\BA-h)}+e^{-i(\BF^T\BA-h)}}{2}\right],\\
=e^{-\BF^\top\BSIGMA\BF}\cos(\BF^T\BA-h),\\
\end{gather}
where we have used the hyperbolic definition of $\cos$ to derive our desired result in the final line. 
\end{document}